\documentclass[runningheads]{llncs}
\usepackage[T1]{fontenc}
\usepackage{graphicx}

\usepackage[margin=2.5cm]{geometry}
\usepackage[title,header]{appendix}

\usepackage{fontawesome5}
\usepackage{amsmath}
\usepackage{amsfonts}
\usepackage{amssymb}

\usepackage{float}

\usepackage{lipsum}
\usepackage[shortlabels]{enumitem}

\newcommand{\DeltaVAE}{$\Delta\!$VAE}
\newcommand{\scalmap}{\Phi_{\mathsf{sc}}}
\newcommand{\spherecoord}{\Phi_{\mathsf{sph}}}

\newcommand\blfootnote[1]{%
  \begingroup
  \renewcommand\thefootnote{}\footnote{#1}%
  \addtocounter{footnote}{-1}%
  \endgroup
}

\begin{document}

\title{Topological degree as a discrete diagnostic for disentanglement, with applications to the {\DeltaVAE}}

\author{Mahefa Ratsisetraina Ravelonanosy\inst{1}\faIcon{envelope} \and
Vlado Menkovski\inst{1,2} \and
Jacobus W. Portegies\inst{1,2}}

\authorrunning{M.R.~Ravelonanosy et al.}

\institute{Department of Mathematics and Computer Science, Eindhoven University of Technology\\ \faIcon{envelope}\email{m.r.ravelonanosy@tue.nl}\\\email{v.menkovski@tue.nl}\\\email{j.w.portegies@tue.nl}\\
\and
EAISI, Eindhoven University of Technology\\
}
\maketitle

\begin{abstract}
We investigate the ability of Diffusion Variational Autoencoder ({\DeltaVAE}) with unit sphere $\mathcal{S}^2$ as latent space to capture topological and geometrical structure and disentangle latent factors in datasets. 
For this, we introduce a new diagnostic of disentanglement: namely the topological degree of the encoder, which is a map from the data manifold to the latent space.
By using tools from homology theory, we derive and implement an algorithm that computes this degree. We use the algorithm to compute the degree of the encoder of models that result from the training procedure.
Our experimental results show that the {\DeltaVAE} achieves relatively small LSBD scores, and that regardless of the degree after initialization, the degree of the encoder after training becomes $-1$ or $+1$, which implies that the resulting encoder is at least homotopic to a homeomorphism.
\end{abstract}



\section{Introduction}
\blfootnote{This work was supported by NWO GROOT project UNRAVEL, OCENW.GROOT.2019.044.}

The Variational Autoencoder (VAE) \cite{kingma2013auto,rezende2014stochastic} and its extensions such as \cite{burgess2018understanding,davidson2018hyperspherical,ijcai2020p375} provide a tool to both embed data into a lower-dimensional latent space via an encoder network and to generate new samples by first sampling in the latent space and by mapping it to the original data-space via a decoder network. The dimension of the latent space is often chosen to be less than the dimension of the dataset. This is partly motivated by the manifold hypothesis \cite{falorsi2018explorations} which states that most high-dimensional data is concentrated near a low-dimensional manifold. The need for discovering low-dimensional representations of a given dataset arises in applications where one wants to make machine learning "easy" for downstream tasks. In other words, the learned latent space, which is the data representation, is intended to be used to ease the training process of machine learning algorithms; hence it should concisely explain the variability in the dataset. 

A desired quality for such a data representation is often that it captures or "disentangles" the explanatory factors of the dataset \cite{bengio2013representation}. In an example, this would mean that for a dataset of pictures of objects rotated over different angles and taken under different lighting conditions, that the rotation angle and the lighting condition can be read off or easily computed from independent parts of latent space. It is difficult to generalize from such examples to give a general definition of disentanglement of latent factors.

Although there is still no agreed formal definition for disentanglement, mathematical definitions do exist. Higgins et al.~\cite{higgins2018towards} introduced two definitions based on group theory: Symmetry Based Disentanglement (SBD) and Linear Symmetry Based Disentanglement (LSBD). Later, Tonnaer et al.~ \cite{tonnaer2022quantifying} converted these definitions for exact disentanglement into a measure to indicate to which extent a representation is disentangled.

Besides formal definitions of disentanglement, one can also formulate mathematical properties that at least reflect aspects of capturing or disentangling latent factors. One could for instance require that nearby points in the dataspace should correspond to nearby points in the latent space representation. This could lead to a requirement that the encoder should be a homeomorphism \cite{falorsi2018explorations,de2018topological}, so that it maps any continuous path in the dataspace into a continuous path in the latent space, and this makes the user aware of the meaning of each directions in the learned latent representation \cite{chadebec2022data}. In order to achieve a homeomorphic encoder, one needs to choose the latent space that matches the topology of the dataset, otherwise one will encounter the manifold mismatch problem \cite{davidson2018hyperspherical}.

Various extensions of VAEs are developed to learn disentangled representations. Some of them such as \cite{burgess2018understanding,cha2023orthogonality,kim2018disentangling} aim to factor the latent space and force independence for each latent factor. Works such as \cite{davidson2018hyperspherical,falorsi2018explorations} intend to choose a particular latent space that matches the topology of data, and give an example of particular data that has that particular topological structure.

In order to have a wider range of latent space and solve the manifold mismatch problem \cite{davidson2018hyperspherical},  P\'erez Rey et al  \cite{ijcai2020p375} developed the Diffusion Variational Autoencoder ({\DeltaVAE}) that allows for any closed Riemannian manifold as latent space. The normalized Riemannian volume is used as prior distribution, and the posteriors are modeled with the heat kernels on the Riemannian manifold.   A good data representation obtained from a {\DeltaVAE} should have at least the following properties: first, the topological structure of the learned latent space should match the topological structure of the dataset. Secondly, the latent space should reflect the symmetry of the dataspace, i.e.~that achieves a low LSBD score \cite{tonnaer2022quantifying}.

In this manuscript, we address several issues.

First, the {\DeltaVAE} has been tested for several manifolds, such as a tori, circles, and projective spaces, but not all of these manifolds were tested with data that would naturally have the corresponding topological structure \cite{alet2021noether,ijcai2020p375}.
For example, a two-dimensional sphere has been used as a latent space for MNIST, although one can hardly argue that the original MNIST data manifold has a spherical topological structure. In this work, we follow up on this exactly in the case of data with a spherical structure.

Second, since at the moment there is no single established mathematical definition for disentanglement, we consider it desirable and necessary to develop a wide range of diagnostics that are somehow related to the intuitive concept of disentanglement.

Third, it turns out that training a {\DeltaVAE}, for instance with $SO(3)$ as a latent space, is relatively difficult: without semi-supervision, disentanglement, as for instance measured with a low LSBD score, only occurs in a fraction of the cases.
We wondered whether this could be related to the initialization and topological obstructions, see also \cite{esmaeili2024topological,falorsi2018explorations}.
Indeed, one could view training of the autoencoder as a deformation, and perhaps even continuous deformation, of the encoder and decoder maps. If this is the case, the topological properties would not change, and if the degree would not be initialized at $1$ or $-1$, the encoder would have no chance to reach suitable disentanglement.
This gave us a second motivation to track the evolution of the degree as a diagnostical tool, and to see if the degree, in practice, changes during training.

In this work, we address these questions for a {\DeltaVAE}  having the sphere $\mathcal{S}^2$ as latent space. Namely, we mathematically assess the ability of {\DeltaVAE} to solve the manifold mismatch problem \cite{davidson2018hyperspherical} and to capture spherical structure in datasets.
\begin{itemize}
    \item We test the {\DeltaVAE} with a two-dimensional spherical latent space with a diagnostic dataset related to spherical harmonics.
    We evaluate its performance according to several metrics, including the LSBD score introduced in \cite{tonnaer2022quantifying}, specified to our case for the group $SO(3)$. We train the {\DeltaVAE} with the semi-supervised LSBD loss function in  \cite{tonnaer2022quantifying}.
    We follow a semi-supervised approach in view of the results in \cite{caselles2019symmetry}, that state that LSBD cannot be attained without some form of supervision.
    
    \item  By using tools from homology theory, we derive and implement an algorithm that computes the topological degree of any smooth map $f$ from $\mathcal{S}^2$ to itself. The developed algorithm is then used to compute the topological degree of the encoder of {\DeltaVAE} having $\mathcal{S}^2$ as latent space.

    \item We use our degree computation algorithm to monitor the evolution of the topological degree of the encoder during the training procedure.
\end{itemize}

Our experiments show that regardless of the initial model weight, the topological degree of the encoder can change to become eventually constant equal to $+1$ or $-1$, after some epochs of the training process. We perform the same experiments for the $S$-VAE \cite{davidson2018hyperspherical} and compare the results. The code that we used in the experiments can be found at \url{https://gitlab.tue.nl/diffusion-vae/degree}.


\section{Related works}
The VAE \cite{kingma2013auto,rezende2014stochastic} and its extensions are among the most used models when it comes to learning disentangled representations \cite{burgess2018understanding,cha2023orthogonality,kim2018disentangling}. Some VAE extensions propose the  use of more complex prior distribution other than the Gaussian in order to better match the distribution of the latent code \cite{hoffman2016elbo,klushyn2019learning,tomczak2018vae,sonderby2016ladder}. Some extensions propose independence of each latent dimension by modifying the VAE loss function \cite{burgess2018understanding,kim2018disentangling}. 
Other extensions use more geometric approaches to make the latent space itself match the geometry of the dataset 
\cite{chadebec2022data,davidson2018hyperspherical,ding2020guided,falorsi2018explorations,huh2024isometric,ijcai2020p375}.

Intuitions and some aspects of disentangled representation are presented in 
\cite{bengio2013representation,do2019theory,van2019disentangled}, while overviews of several disentanglement metrics are given in \cite{carbonneau2022measuring} and \cite{sepliarskaia2019not}.
Disentanglement is originally assessed with visual inspections and performance on downstream tasks \cite{carbonneau2022measuring}. Efforts have been devoted to propose metrics to evaluate different aspects of disentanglement  \cite{esmaeili2024topological,higgins2017beta,locatello2019challenging,suter2019robustly,tonnaer2022quantifying}. The  disentanglement metrics derived in these works do not check geometric aspects of disentanglement such as homeomorphism and topological degree according to the original mathematical definitions of these aspects.
The degree \emph{was} mentioned as a topological obstruction to homeomorphic autoencoding in \cite{esmaeili2024topological}.

\section{Topological degree as a diagnostic for disentanglement}
In this section we introduce topological degree as a discrete diagnostic for disentanglement. It can be seen as a weakening of a check for homeomorphism. Indeed, in practice to check whether a continuous map is a homeomorphism or not. But at least it is known that homeomorphisms have topological degree $\pm{1}$, and Hopf's Theorem \cite[Page 51]{milnor1965topology} implies that a smooth map $f: \mathcal{S}^2\to \mathcal{S}^2$ has topological degree $\pm{1}$ if and only if $f$ is homotopic to a homoemorphism. Hence, having an encoder of degree $\pm{1}$ indicates that the encoder is at least homotopic to a homeomorphism; and having an encoder with degree other than $+1$ and $-1$ indicates in particular that the encoder is not a homeomorphism, thus it does not preserve the topology of the dataset. In this section we introduce the degree of the encoder $h: \mathcal{X}\to Z$ from a data manifold $\mathcal{X}$ to latent space $Z$.

\subsection{Topological degree}
In words, the topological degree of the encoder restricted to the data manifold, which is a continuous map $h: \mathcal{X}\to Z$ is an integer that represents the number of times that $h$ wraps the data manifold $\mathcal{X}$ around the latent space $Z$ (cf. \cite[Page 134]{hatcher2005algebraic} and \cite[Page 27]{milnor1965topology}). The topological degree can be positive of negative integer, depending on the orientation of $h(\mathcal{X})$ and $Z$.

\subsection{Computing the topological degree}
Although general methods exist for the computation of the degree and other properties of homology groups \cite{kaczynski2004computational}, we developed and implemented a basic algorithm targeted to the case at hand of computing the degree of a map between spheres.
Our algorithm relies on the following steps for a given function $f:\mathcal{S}^2\to\mathcal{S}^2$.

\begin{itemize}[align=left]
    \item[Step 1] We fix two suitable triangulations $T(n)$ and $T(k=3)$ on the sphere in the domain and the sphere in the codomain of $f$ respectively. The parameter $n$ is a measure of how fine the triangulation is, and needs to be chosen depending on $f$.
    \item[Step 2] Given $f$, we construct a smooth function $g:\mathcal{S}^2\to \mathcal{S}^2$ such that $\vert\vert g-f\vert\vert_{\infty}<\pi.$ It follows that $f$ and $g$ are homotopic, and therefore have the same degree.
    \item[Step 3] We construct a chain map $\hat{g}$, which when interpreted as a chain map from simplicial chains to singular chains, is chain homotopic to $g_\#$, the chain map induced by the continuous map $g$. The degree of $\hat{g}$ is then equal to the degree of $g$ and $f$.
    \item[Step 4] We finally numerically compute the degree of $\hat{g}$, and it corresponds to the degree of $f$.
\end{itemize}

The steps are worked out in Appendix \ref{topo_deg_section}.

\section{Experiments}

We evaluate the ability of {\DeltaVAE} to capture topological structure in dataset when the known generating factors have the topological structure of a two-dimensional sphere. Natural datasets with such latent structure are given by pictures of axisymmetric objects rotated over various angles, or pictures of an axisymmetric picture on the sphere taken from different angles. In the latter case, such pictures can be viewed as real-valued functions defined on the unit sphere $\mathcal{S}^2$ that are axisymmetric about some axis, which in turn can be expressed as linear combination of real spherical harmonics of degree $L$ \cite[Page 88]{brocker2013representations} \cite{blanco1997evaluation} for a fixed $L>0.$ This is the basis of our diagnostic dataset.

\subsubsection{Data} For an fixed odd integer $L\geq 3$, we start with the real spherical hamonic $Y_0^L$ of degree $L$ and order $0$, which gives an axisymmetric colouring of $\mathcal{S}^2$. To generate the dataset, we then sample uniformly $4266$ group elements in $SO(3)$, let them act on $Y_0^L$ and express the resuling functions in their coordinates in the basis  $<Y_{-L}^L,\dots, Y_L^L>$ of real spherical harmonics of degree $L$ \cite{blanco1997evaluation,harmonicsclaus}. The resulting dataset $(x_i)_i$ is then a subset of $\mathbb{R}^{2L+1}$.

\subsubsection{The models} We the train the {\DeltaVAE} \cite{ijcai2020p375}, and compare the result to the $S$-VAE \cite{davidson2018hyperspherical}. Since the ground truth generating factor of our dataset is homeomorphic to $\mathcal{S}^2$, the $2$-dimensional sphere is used as latent space in both models. We asymptotically approximate the KL-term in the loss of the {\DeltaVAE} up to and including the term with $t^2$, following to \cite{menkovski2024small}.

\subsubsection{Training} We train the {\DeltaVAE} and $S$-VAE with a semisupervised LSBD-loss as in \cite{tonnaer2022quantifying}, except that we do not alternate between supervised and unsupervised training but rather in every training step consider a batch of data with and a batch of data without labels. Just like in \cite{tonnaer2022quantifying}, instead of optimizing the infimum in the LSBD score over all representations, we use a trivial upper bound which involves one (in our case the identity) representation. The ratio between data with and data without labels is $0.5$. We also add a semisupervised LSBD loss for the decoder. For $L=5,7,9$ we train with $600$ epochs, and for $L=11$ we train with $1200$ epochs.

\subsubsection{LSBD score} In addition to the topological degree, we evaluate the LSBD score outlined in \cite{tonnaer2022quantifying} with the group $SO(3)$. The dataset is generated via the natural action of $SO(3)$ on spherical harmonics. That action should correspond to a linear action of $SO(3)$ on $\mathbb{R}^3$, which preserves the unit sphere $\mathcal{S}^2$. The representation of the data given by the models is then good if the corresponding LSBD score is small.

\subsubsection{Further metrics} 
Furthermore, we compute the distance distortion metric as given in \cite{ijcai2020p375}, and the log-likelihood estimate as in \cite{DBLP:journals/corr/BurdaGS15}; for further details see also \cite{ijcai2020p375}.

\subsubsection{Experimental results} We train the {\DeltaVAE} with spherical harmonics dataset of degree $L=5, 7, 9, 11.$ For each of these values of $L$, each model is trained $5$ times. The model weights are initiated randomly according to PyTorch, but at the end of each training, the encoder of the resulting models reach degree $\pm{1}$.  The numerical results of the experiments are presented in Table \ref{result}, where the absolute value of the degree is reported.

\begin{table}[H]
\caption{Results for training the {\DeltaVAE} and the $S$-VAE. The "degree" column reports how often the absolute value of the degree equaled $1$ after training.}
\vspace{0.5cm}
\centering
\begin{tabular}{|c|c|c|c|c|c|c|c|}
\hline
\textbf{Model} & \textbf{LL} & \textbf{ELBO} & \textbf{KL} & \textbf{RE} & \textbf{Distortion} & \textbf{Final degree} & \textbf{LSBD} \\
\hline
\multicolumn{8}{|c|}{\textbf{Spherical harmonics of degree $L= 5$}} \\
\hline
{\DeltaVAE} & $-15.92\pm{0.03}$  & $6.74\pm{0.01}$  & $6.70\pm{0.00}$  & $0.005\pm{0.001}$  & $0.05\pm{0.009}$  & $5$ out of $5$  & $0.013\pm{0.012}$ \\
\hline
$S$-VAE & $-0.222\pm{0.001}$  & $8.41\pm{0.08}$  & $8.39\pm{0.08}$  & $0.023\pm{0.003}$  & $0.001\pm{0.000}$ & $5$ out of $5$ & $0.002\pm{0.000}$  \\
\hline
\multicolumn{8}{|c|}{\textbf{Spherical harmonics of degree $L=7$}} \\
\hline
{\DeltaVAE} & $-19.72\pm{0.03}$  & $6.96\pm{0.03}$  & $6.70\pm{0.00}$  & $0.27\pm{0.03}$  & $0.05\pm{0.02}$ & $5$ out of $5$ & $0.12\pm{0.06}$ \\
\hline
$S$-VAE & $-0.26\pm{0.00}$  & $8.33\pm{0.03}$  & $7.90\pm{0.03}$ & $0.43\pm{0.03}$  & $0.09\pm{0.01}$ & $5$ out of $5$ & $0.20\pm{0.03}$ \\
\hline
\multicolumn{8}{|c|}{\textbf{Spherical harmonics of degree $L=9$}} \\
\hline
{\DeltaVAE} & $-23.32\pm{0.02}$  & $6.821\pm{0.004}$  & $6.693\pm{0.000}$  & $0.128\pm{0.004}$  & $0.010\pm{0.007}$ & $5$ out of $5$ & $0.008\pm{0.003}$ \\
\hline
$S$-VAE & $-0.31\pm{0.00}$  & $8.23\pm{0.04}$  & $7.57\pm{0.31}$ & $0.66\pm{0.02}$  & $0.12\pm{0.01}$ & $1$ out of $5$ & $0.29\pm{0.02}$ \\
\hline
\multicolumn{8}{|c|}{\textbf{Spherical harmonics of degree $L=11$}} \\
\hline
{\DeltaVAE} & $-33.06\pm{0.05}$ & $12.90\pm{0.04}$  & $12.69\pm{0.00}$ & $0.18\pm{0.04}$ & $0.05\pm{0.02}$ & $5$ out of $5$ & $0.09\pm{0.04}$ \\
\hline
$S$-VAE & $-0.36\pm{0.00}$  & $9.07\pm{0.10}$  & $8.33\pm{0.11}$  & $0.73\pm{0.03}$  &  $0.15\pm{0.03}$ & $1$ out of $5$ & $0.35\pm{0.05}$ \\
\hline
\end{tabular}
\label{result}
\end{table}

\subsubsection{Evolution of the degree during the training} In order to get insight into the evolution of the degree during the training, we conducted more experiments for spherical harmonics of degree $L=7, 5, 3$ with {\DeltaVAE}. For each $L$, we performed $5$ experiments in which we recorded the degree before and after training. For the five experiments with $L=7$, three of the initial models have encoder of degree $0$, one with degree $2$ and one with degree $-1$. In the five experiments where $L=5$, the degree of the initial models turned out to be $0$. For $L=3,$ four of the initial models have encoder of degree $0$, and one has encoder of degree $2$. Whereas the absolute value of the  degree after all training was $1$. In particular, even though we share the opinion that topological obstructions might hamper training \cite{esmaeili2024topological,falorsi2018explorations}, for the {\DeltaVAE} the obstruction to the degree can be overcome.

\section{Discussion}

We derive a second order expansion of the heat kernel on the unit sphere $\mathcal{S}^2$ by using the theoretical result of \cite{menkovski2024small}, and use it as approximation in the {\DeltaVAE} loss function. The effect of such higher order approximation in the performance of {\DeltaVAE} is not studied yet. In fact, to guarantee robustness of our algorithm, we limit the possible values of the heat kernel time $t$ by using a sigmoid activation function. Sometimes it requires careful tuning of the parameters to not have $t$ be limited by one of its boundaries.

Our algorithm for degree computation could be generalized to higher dimensional sphere $\mathcal{S}^d$ with $d > 2$, but due to the curse of dimensionality, practical computation is most likely only feasible in very low dimensions: for a $d$-dimensional manifold and a discretization length $\delta$, the number of faces needed in the triangulation scales as $\delta ^{-d}$.

Our semi-supervised approach is inspired by the result of \cite{caselles2019symmetry} which states that LSBD cannot be inferred without any supervision, and this approach was also used in \cite{tonnaer2022quantifying}. Note that the amount of semisupervision is relatively high in our experiments. For lower degree spherical harmonics ($L = 1, 3, 5$), the amount of semisupervision can be reduced drastically, although we have not yet performed a systematic study.

\section{Conclusion}
We evaluate the {\DeltaVAE} with spherical latent space using a diagnostic dataset that arises from the irreducible action of $SO(3)$ on spherical harmonics. In particular, we evaluate to what extent it can capture the topological properties or disentangle the generating factors of the underlying dataset, as measured by the LSBD score, and as expressed by a new discrete diagnostic for disentanglement: the degree of the encoder. We also use the encoder degree as a means to gain more insight in the training behavior.

First, we obtain relatively small LSBD scores, which expresses that the {\DeltaVAE} indeed can capture or disentangle the latent rotational factor relatively well. In comparison with the $S$-VAE, we find that the $S$-VAE typically obtains better log-likelihood scores, while the reconstruction error and LSBD score are a bit better for the {\DeltaVAE}.

Secondly, we implemented an algorithm for computing the topological degree of the encoder and find that even though the encoder is typically initialized with degree $0$, this degree can change and after training the encoder indeed has degree of $\pm{1}$, which means by Hopf's Theorem that the encoder is at least homotopic to a homeomorphism and that the learned spherical representation preserves the topological structure of the dataset at least up to a homotopy. In particular, we find that the sphere in latent space is completely covered by the image of the data manifold.

\appendix
\section{Numerical computation of degree}\label{topo_deg_section}
\subsection{Step 1: Triangulate spheres}

In this section, we describe how we triangulate the sphere $\mathcal{S}^2$. More precisely, for every $n \in \mathbb{N}$, we endow the sphere $\mathcal{S}^2$ with a $\Delta$-complex structure \cite[Page 103]{hatcher2005algebraic}. 

The intuition behind the construction is simple: we make a regular grid in spherical coordinates, except we identify all points with $\phi = 0$ (the North Pole) and all points with $\phi = \pi$ (the South Pole). We divide all squares that we obtain this way by a diagonal that runs from the down-left corner to the top-right corner in spherical coordinates.

What follows in this section is a precise description of the exact $\Delta$-complex structure that we describe here for completeness. The construction continues in Step 2.

As of \cite[Page 103]{hatcher2005algebraic}, for $n \in \mathbb{N}$ we define the standard $n$-simplex by
\begin{equation}
\label{standard simplex}
\mathbb{D}^n := \left\{ (t_0, \dots, t_n) \in \mathbb{R}^{n+1} \ \middle| \ \sum_{i=0}^n t_i = 1 \text{ and for all i, } t_i \geq 0 \right\}.
\end{equation}
We define the map of spherical coordinates
$\spherecoord : \mathbb{R}^2 \to \mathcal{S}^2$ by
\[
\spherecoord (\theta, \phi) := (\sin \phi \cos \theta, \sin \phi \sin \theta, \cos \phi).
\]
Moreover, define the scaling map
$\scalmap : \mathbb{R}^2 \to \mathbb{R}^2$ by
\[
\scalmap(i, j) := \left( \frac{2 \pi i}{n}, \frac{(j+1) \pi }{n+1} \right).
\]

Finally, for a given map $\sigma:\mathbb{D}^n\to X$ from $\mathbb{D}^n$ to a topological space $X$, the boundary $\partial\sigma$ is defined to be the restriction of $\sigma$ on the boundary of $\mathbb{D}^n$ cf. \cite[Page 105]{hatcher2005algebraic}.

\subsubsection{Vertices as a family of maps}\label{0simplices}
We now construct a family of maps from $\{\mathbb{D}^0\}$ to $\mathcal{S}^2$, where $\mathbb{D}^0$ denotes the standard $0$-simplex cf. Equation (\ref{standard simplex}).

Let $n\geq 3$ be fixed. Define first the following index set
\begin{equation}\label{vertex-index}
\mathcal{V}_n:=\Bigl\{(0,-1), (0,n)\Bigr\} \bigcup \Bigl\{(i,j):\; i=0,1,\dots,n-1\;\mathrm{and}\;j = 0,1,\dots, n-1\Bigl\}. 
\end{equation}
For $\alpha \in \mathcal{V}_n$, we define $\Delta^\alpha : \mathbb{D}_0 \to \mathcal{S}^2$ by
\[
\Delta^\alpha(x) :=  \spherecoord \circ \scalmap (\alpha).
\]

\subsubsection{Construction of edges as a family of maps}\label{edge_as_map}
Next, we construct a family of maps from the standard $1$-simplex $\mathbb{D}^1$ of Equation (\ref{standard simplex}) to $\mathcal{S}^2$. Fix $n\geq 3$ and let us define the following index set.
\begin{equation}\label{edges-index}
    \begin{split}
         \mathcal{E}_n &:=\Bigl\{\left((0,-1), (i\%n,0) \right):\; i=0,1,\dots,n\Bigr\}\bigcup \Bigl\{\left((i,n-1), (0,n) \right):\; i=0,1,\dots,n-1\Bigr\}\\
   &\quad \;\; \bigcup\; \Bigl\{\left((i,j), ((i+1)\%n,j) \right):\; i=0,1,\dots,n-1\quad \mathrm{and}\; j = 0,1,\dots,n-1\Bigr\}\\
   &\quad\;\; \bigcup \; \Bigl\{\left((i,j), (i,j+1) \right):\; i=0,1,\dots,n-1\quad\mathrm{and}\; j = 0,1,\dots,n-2\Bigr\}\\
   &\quad \;\; \bigcup \; \Bigl\{\left((i,j), ((i+1)\%n,j+1) \right):\; i=0,1,\dots,n-1\quad\mathrm{and}\; j = 0,1,\dots,n-2\Bigr\}.
   \end{split}
    \end{equation}.

Now, let $\tau = (\tau_1,\tau_2)\in \mathcal{E}_n.$ Then $\tau$ indicates an edge $\Delta^\tau: \mathbb{D}^1 \to S^2$ in the following manner.
Denote by $L_\tau: \mathbb{R}^2 \to \mathbb{R}^2$ the (unique) linear map such that $L_\tau(e_0) = \tau_1$ and $L_\tau(e_1) = \tau_2$. Then $\Delta^\tau : \mathbb{D}^1 \to \mathcal{S}^2$ is defined as
\[
\Delta^\tau(x) := (\spherecoord \circ \scalmap \circ L_\tau)(x).
\]

\subsubsection{$\Delta$-complex $T(n)$ on $\mathcal{S}^2$}\label{face_as_map}
Here we construct a {$\Delta$-complex structure cf. \cite[Page 103]{hatcher2005algebraic} on the unit sphere. We start by constructing a family of smooth maps from the standard $2$-simplex $\mathbb{D}$ to $\mathcal{S}^2$. We are going to use the notations from Subsection \ref{0simplices} and Subsection \ref{edge_as_map}. We first construct an index set by mean of the sets in Equation (\ref{vertex-index}) and in Equation (\ref{edges-index}). More precisely, fix $n\geq 3$ and consider the following index set
\begin{equation}\label{face-index-plus}
    \begin{split}
        \mathcal{F}_n^+&:=\Bigl\{ ( (i, j), ((i+1)\%n, j), ((i+1)\%n, j+1) ):\; j=0,1,\dots,n-2\;\mathrm{and}\; i=0,1,\dots,n-1 \Bigr\}\\
        &\;\quad\;\bigcup \;\Bigl\{( (i, n-1), ((i+1)\%n, n-1), (0, n) ):\;i=0,1,...,n-1\Bigr\},
    \end{split}
\end{equation}

\begin{equation}\label{face-index-minus}
    \begin{split}
        \mathcal{F}_n^-&:= \Bigl\{ ( (i, j), (i, j+1), ((i+1)\%n, j+1)):\;j =0,1,\dots,n-2\;\mathrm{and}\;i=0,1,\dots,n-1 \Bigr\}\\
        &\;\quad\;\bigcup\; \Bigl\{ ( (0, -1), (i, 0), ((i+1)\%n, 0)):\;i=0,1,\dots,n-1 \Bigr\},
    \end{split}
\end{equation}
where $i\%n$ denotes $i$ modulo $n$,
and
\begin{equation}
\mathcal{F}_n := \mathcal{F}_n^+ \cup \mathcal{F}_n^-.
\end{equation}

Let $\tau = (\tau_0,\tau_1,\tau_2) \in \mathcal{F}_n$. Then $\tau$ indicates a face $\Delta^\tau : \mathbb{D}^2 \to \mathcal{S}^2$ in the following manner. We denote by $U_\tau : \mathbb{R}^3 \to \mathbb{R}^3$ the (unique) linear map such that $U_\tau(e_i) = \tau_i$, for $i=0, 1, 2$. Then $\Delta^\tau : \mathbb{D}^2 \to \mathcal{S}^2$ is defined as
\[
\Delta^\tau(x) := (\spherecoord \circ \scalmap \circ U_\tau)(x).
\]

Note that the maps $\Delta^\tau$ for $\tau \in \mathcal{F}_n^+$ have opposite orientation from the maps $\Delta^\tau$ for $\tau \in \mathcal{F}_n^-$.

We then have a family of functions 
\[
\{ \Delta^\tau \ | \ \tau \in \mathcal{V}_n \cup \mathcal{E}_n \cup \mathcal{F}_n \}
\]
satisfying the conditions for a $\Delta$-complex structure \cite[(i),(ii),(iii), Page 103]{hatcher2005algebraic}.

For a given $n$, we denote by $T(n)$ the $\Delta$-complex structure defined by this family of functions.

\begin{remark}{(Important)}\label{homology_basis_remark}
As mentioned earlier, the maps $\Delta^{\tau}$'s for $\tau\in\mathcal{F}_n^+$ have opposite orientation from the maps $\Delta^{\tau}$'s for $\tau\in\mathcal{F}_n^.$
This implies that the homology class of  $$\sum_{\tau\in \mathcal{F}_n}\Delta^{\tau}$$ does generate the homology group of $\mathcal{S}^2$ since $\mathcal{S}^2$ is an orientable surface.

Instead, a generator of the homology group of $\mathcal{S}^2$ that we are going to use is the homology class of  
\begin{equation}\label{unit_homology}
    \sum_{\tau\in \mathcal{F}_n}i_{\tau}\Delta^{\tau}
\end{equation}
with \begin{equation}
    i_{\tau} = \begin{cases}
        -1\;&\text{if}\; \tau \in \mathcal{F}_n^-\\
        1 \;&\text{if}\;\tau\in \mathcal{F}_n^+.
    \end{cases}
\end{equation}
\end{remark}

\subsubsection{Geometric interpretation, initial setting and preliminary result}\label{one_hot_encoding}
Geometrically speaking, we just give a triangulation of the sphere: the set of vertices is given by $\{\Delta^{\tau}(\mathbb{D}^0):\;\tau\in \mathcal{V}_n\}$, the set of edges is given by $\{\Delta^{\tau}(\mathbb{D}^1):\; \tau\in \mathcal{E}_n\}$, and the set of faces is given by $\{\Delta^{\tau}(\mathbb{D}^2):\; \tau \in \mathcal{F}_n\}.$

Throughout the rest of this work, we fix $k = 3$ and we use the $\Delta$-complex structure $T(3)$ on the codomain of the function $f$.

\subsection{Step 2: Construction of a map $g$ such that $\|g - f\|_\infty < \pi$}

Now, we are going to build a continuous function $g: \mathbb{S}^2\to \mathbb{S}^2$ that is homotopic to $f$.

\subsubsection{Definition of $g$ on the set of vertices}\label{g_on_vertices}
Let us define the function $g$ on the set of vertices $\bigcup_{\tau\in \mathcal{V}_n}\Delta^{\tau}(\mathbb{D}^0).$ For a vertex $x\in \bigcup_{\tau\in \mathcal{V}_n}\Delta^{\tau}(\mathbb{D}^0),$ we choose an element $y\in \bigcup_{\tau\in \mathcal{V}_3}\Delta^{\tau}(\mathbb{D}^0)$ that is closest to $f(x)$ in spherical coordinates, and we define
$$g(x):= y.$$
Note that this comes down to independently rounding the spherical coordinates to values that are in the grid.

\subsubsection{Designated paths between two points on the sphere}
Before we describe how $g$ is defined on edges, let us first indicate some designated paths on the sphere. For $p, q \in \mathcal{S}^2$, we define a designated path $e_{p, q} : \mathbb{D}^1 \to \mathcal{S}^2$ from $p$ to $q$ as follows. We denote by $(\theta_p, \phi_p)$ and $(\theta_q, \phi_q)$ the spherical coordinates of $p$ and $q$ respectively. In spherical coordinates, the path $e_{p, q}$ has constant speed and consists of at most one vertical and at most one horizontal segment. The $\theta$-coordinate of the vertical segment is
\begin{itemize}
    \item $\theta_p$ if $0 < \phi_p < \phi_q$
    \item $\theta_q$, if $0 < \phi_q \leq \phi_p$
    \item $0$ if either $\phi_q = 0$ or $\phi_p = 0$.
\end{itemize}
The $\phi$-coordinate of the horizontal segment coincides with $\max(\phi_p, \phi_q)$. There are then two ways to connect the points with coordinates 
\[
(\theta_p , \max(\phi_p, \phi_q)) \text{ and } (\theta_q, \max(\phi_p, \phi_q))
\]
with a horizontal segment, and we choose the shortest option (which is unique because $k=3$).

\subsubsection{Definition of $g$ on the $1$-skeleton $\bigcup_{\tau\in \mathcal{E}_n}\Delta^{\tau}(\mathbb{D}^1)$ of $T(n)$}\label{edge_image}
We then define $g$ on the edge $\Delta^\tau(\mathbb{D}^1)$, with $\tau \in \mathcal{E}_n$, by
\begin{equation}\label{g_values_on_edges}
g(x) := e_{g(\tau_1), g(\tau_2)} \circ (\Delta^\tau)^{-1}(x)
\end{equation}
for $x \in \Delta^\tau (\mathbb{D}^1)$.

Note that this gives a continuous function $g$ defined on the $1$-skeleton $\bigcup_{\tau\in \mathcal{E}_n}\Delta^{\tau}(\mathbb{D}^1)$ of $T(n)$. 

Let us end this subsection by proving some properties of $g$ on $\bigcup_{\tau\in \mathcal{E}_n}\Delta^{\tau}(\mathbb{D})$. We are going to use the following definition.
\begin{definition}{(Timezone)}\label{timezone}
Let $t\in \{0,1,2\}$. A timezone $\mathcal{T}_t$ of $\mathcal{S}^2$ defined by the triangulation $T(3)$ is the set of points $x\in \mathcal{S}^2$ with spherical coordinates $(\theta, \varphi,1)$ such that the azimuth angle $\theta$ satisfies  $$\frac{2t\pi}{3}\leq \;\theta\leq \; \frac{2(t+1)\pi}{3}.$$
\end{definition}

We have the following lemma.
\begin{lemma}\label{g_boundary_image}
Define $\epsilon := \sqrt{3} \sin(\pi/8) > 0.66$. Let $L_f$ be a Lipschitz constant of $f$, and let $N$ such that for all $n\geq N$ we have 
$$\mathrm{diam}(\Delta^{\tau}(\mathbb{D}^2))< \frac{\epsilon}{L_f}$$ for all $\tau \in \mathcal{F}_n$.

Then for any $n\geq N$ and for any $\tau\in \mathcal{F}_n$, the image by $g$ of the boundary of $\Delta^{\tau}(\mathbb{D}^2)$ is included in one timezone of $T(3)$ cf. Definition \ref{timezone}.
\end{lemma}

\begin{proof}
Let $n\geq N$ and let $\tau\in \mathcal{F}_n$. We have by assumption 
$$\mathrm{diam}\Bigl(f\left(\Delta^{\tau}(\mathbb{D}^2) \right)\Bigr) < \epsilon.$$

Note that by the rounding procedure in the definition of $g$, the result follows if for some $\theta_0 \in \mathbb{R}$, the set $f\left(\Delta^{\tau}(\mathbb{D}^2) \right)$ is contained in the set
\[ A_{\theta_0} := \mathcal{S}^2 \setminus \{ (\cos \theta \sin \phi, \sin \theta \sin \phi, \cos \phi) \ | \ 0 \leq \theta - \theta_0 \leq 2 \pi / 3, \pi/8 \leq \phi \leq 7 \pi / 8 \}. \]
But if the diameter of $f\left(\Delta^{\tau}(\mathbb{D}^2) \right)$ is smaller than $\epsilon$, such a $\theta_0$ can always be found. Indeed, by symmetry we can assume that
\[
\inf\{ \theta \in [0, 2\pi) \ | \ (\cos \theta \sin \phi, \sin \theta \sin \phi, \cos \phi) \in f\left(\Delta^{\tau}(\mathbb{D}^2)\right) \text{ for some } \pi/8 \leq \phi \leq 7 \pi/8 \} = 0.
\]
Then $f\left(\Delta^{\tau}(\mathbb{D}^2) \right)$ is contained in $A_0$, since $\epsilon$ is exactly the Euclidean distance between the points with spherical coordinates $(0, \pi/8)$ and $(2\pi/3, \pi/8)$.

\end{proof}

\subsubsection{Definition of $g$ on $\mathcal{S}^2$ }\label{homotopic_map_g}
Consider $n\geq N$ where $N$ is specified in Lemma \ref{g_boundary_image}. Let us define a continuous function $g: \mathcal{S}^2\to \mathcal{S}^2$ which is homotopic to $f$ and such that the restriction of $g$ on the $1$-skeleton $\bigcup_{\tau\in \mathcal{E}_n}\Delta^{\tau}(\mathbb{D}^1)$ is defined in Subsection \ref{edge_image}.

Let $x\in \mathcal{S}^1$ such that $x\in \mathrm{Int}\Bigl(\Delta^{\tau}(\mathbb{D}^2) \Bigl)$ for some $\tau\in \mathcal{F}_n.$ Assume that $\tau = (\tau_0,\tau_1,\tau_2)$ and the boundaries of $\Delta^{\tau}(\mathbb{D}^2)$ are given by   $\Delta^{\tau_i}(\mathbb{D}^1)$'s, where $\tau_i\in \mathcal{E}_n$ for $i=0,1,2.$ Let $x_i\in \Delta^{\tau_i}(\mathbb{D}^1)$ be the geodesic projection of $x$ on the edge $\Delta^{\tau_i}(\mathbb{D}^1)$ and let $\alpha_i(x):= \mathrm{Dist}(x,x_i)$ be the geodesic distance between $x$ and $x_i$ for $i=0,1,2.$

Since $x$ is an interior point of $\Delta^{\tau}(\mathbb{D}^2)$, then $\alpha_i(x)>0$ for $i=0,1,2$. Furthermore, Lemma \ref{g_boundary_image} implies that 
$$\vert\vert \sum_{i=0}^2\frac{1}{\alpha_i(x)}\;g(x_i)\vert\vert \neq 0.$$

Therefore, we define the continuous function $g$ for $x\in \mathcal{S}^2$ by
\begin{equation}\label{the_function_g}
 \displaystyle g(x) :=\begin{cases}
 e_{g(\tau_1), g(\tau_2)} \circ (\Delta^\tau)^{-1}(x) &\mathrm{if} \; x\in \bigcup_{\tau\in \mathcal{E}_n}\Bigl(\Delta^{\tau}(\mathbb{D}^1)\Bigr)\\
 \frac{\sum_{i=0}^2\frac{1}{\alpha_i(x)}\;g(x_i)}{\vert\vert \sum_{i=0}^2\frac{1}{\alpha_i(x)}\;g(x_i)\vert\vert} &\mathrm{if}\; x\in \bigcup_{\tau\in \mathcal{F}_n}\mathrm{Int}\Bigl(\Delta^{\tau}(\mathbb{D}^2)\Bigr),
 \end{cases}
\end{equation}
where $e_{g(\tau_1), g(\tau_2)} \circ (\Delta^\tau)^{-1}(x)$ is defined in Equation (\ref{g_values_on_edges}).

\subsubsection{Proof that $g$ and $f$ are homotopic}

We have the following property of the constructed continuous function $g$.

\begin{lemma}\label{g_continuous} 
Let $\epsilon,\; N$ and $L_f$ be as specified in Lemma \ref{g_boundary_image}.
Then, for all $n\geq N$, we have 
$$\mathrm{d} (f(x),g(x))< \pi$$ for all $x\in \mathcal{S}^2$, where $\mathrm{d}$ denotes the geodesic distance on $\mathcal{S}^2$.
\end{lemma}

\begin{proof}
This follows by construction of $g$.
\end{proof}

The fact that $f$ and $g$ are homotopic now follows from the following lemma.

\begin{lemma}\label{degree_lemma}
Let $f,g: \mathcal{S}^2\to \mathcal{S}^2$ be  smooth maps such that 
$$\| f - g \|_\infty := \sup_{x\in \mathcal{S}^2}\{\mathrm{d}(f(x),g(x))\}< \pi,$$
where $\mathrm{d}$ is the geodesic distance on $\mathcal{S}^2$. Then
$$ \mathrm{deg}(f) = \mathrm{deg}(g).$$
\end{lemma}

\begin{proof}
For all $x$, there exists an unique geodesic joining $f(x)$ and $g(x)$ by assumption. The result  follows by using \cite[12.1.2. Theorem]{dubrovin2012modern} and  Hopf's theorem cf. \cite[Page 51]{milnor1965topology}.

\end{proof}

\subsection{Step 3: Construction of a chain map $\hat{g}$ }\label{combinatorial}
In this subsection construct a chain map
\begin{equation}\label{chainmap}
\hat{g}: C_\bullet(T(n)) \to C_\bullet(T(3)),
\end{equation}
where $$C_\bullet(T(n))$$ denotes the simplicial chain complex. Intuitively, $\hat{g}$ just corresponds exactly to $g$ on vertices and edges. The details are below.

We first define $\hat{g}$ on vertices. More precisely, we define $$\hat{g} : C_0(T(n)) \to C_0(T(3))$$ by
\[
\hat{g}(\Delta^\tau) := g \circ \Delta^\tau,
\]
where $$C_0(T(n)):= <\Delta^{\tau}:\; \tau \in \mathcal{V}_n>$$ is the Abelian free group generated by the vertices.
Note that by definition of $g$, there exists a $\rho \in \mathcal{V}_3$ such that $g \circ \Delta^\tau = \Delta^\rho$.

We now define $\hat{g}$ on edges. Let $\tau \in \mathcal{E}_n$. Then there exist $\tau^1, \dots, \tau^\ell \in \mathcal{E}_3$ such that
\[
(g \circ \Delta^\tau)(\mathbb{D}^1) = \bigcup_{i=1}^\ell \Delta^{\tau^i}(\mathbb{D}^1)
\]
and there exist $\lambda_1, \dots, \lambda_\ell \in \{ -1, 1\}$ such that

\[
\partial \left( \sum_{i=1}^\ell \lambda_i (g \circ \Delta^{\tau^i}) - g \circ \Delta^\tau \right) = 0
\]
in singular homology. We define
\[
\hat{g}(\Delta^\tau) := \sum_{i=1}^\ell \lambda_i (g \circ \Delta^{\tau^i}).
\]

Finally, we define $\hat{g}$ on faces. For $\tau = (\tau_0,\tau_1,\tau_2)\in \mathcal{F}_n$ with $\tau_i\in \mathcal{V}_n$ ($i=0,1,2$), we construct $\hat{g}(\Delta^{\tau})$ as the element in $C_2(T(3)):= <\Delta^{\tau}:\; \tau\in \mathcal{F}_n>$ such that 
\[
\partial (\hat{g}(\Delta^\tau)) = 
\hat{g}(\partial \Delta^\tau)
\]
and such that the ``word norm" of $\hat{g}(\Delta^{\tau})$ is minimal in the free Abelian group $c_2(T(3))$. Note that such a minimal element exists because $g(\partial \Delta^\tau)$ is contained in a timezone.

Note that with these definitions, $\hat{g}$ is indeed a chain map, cf.~\cite[Page 11]{hatcher2005algebraic}. The construction of $\hat{g}$ makes it chain-homotopic to $g_\#$, which is the chain map induced by $g$ from simplicial to singular chains \cite[Page 111]{hatcher2005algebraic}.

\begin{corollary}\label{zero_homology}
Let $N$ be as in Lemma \ref{g_continuous} and let $n\geq N$.
Then the chain maps $g_\#$ defined by $g$ and the chain map $\hat{g}$, defined as maps from simplicial chains to singular chains, are chain homotopic.
 
\end{corollary}

\begin{proof}

By possibly subdividing the simplices $g_\# (\Delta^\tau)$ for some $\tau \in \mathcal{F}_n$, we obtain a chain map $h$ that is chain homotopic to $g_\#$, and for which for all $\tau \in \mathcal{F}_n$, $h(\Delta^\tau)$ and $\hat{g}(\Delta^\tau)$ have the same boundary.

It now suffices to show that there exists a chain homotopy $P$ between $h$ and $\hat{g}$. A chain homotopy $P$ is a sequence of maps $$P : C_j(T(n)) \to C_{j+1}(\mathcal{S}^2):= \; <\sigma: \mathbb{D}^{j+1}\to \mathcal{S}^2:\;\sigma\;\text{is continuous}>$$ such that
\[
\partial \circ P + P \circ \partial = h - \hat{g}
\]
We may define $P = 0$ on $C_n$ for $n \neq 2$. We will now define $P$ on $C_2( K(n))$.

Let $\tau \in \mathcal{F}_n$. Then, the singular chains $h(\Delta^\tau)$ and $\hat{g}(\Delta^\tau)$ can be considered as singular $2$-chains in the same punctured sphere which is homeomorphic to the disc, and the second homology group of the disc is trivial. Therefore, there exists a $3$-chain $P(\Delta^\tau)$ such that
\begin{equation}\label{chain_h_for_g}
\partial P (\Delta^\tau) = h(\Delta^\tau) - \hat{g}(\Delta^\tau)
\end{equation}
which can be trivially extended to a chain homotopy between $h$ and $\hat{g}$.
\end{proof}

\section{Estimating a Lipschitz constant of a neural network}
In this section, we give an estimation of a Lischitz constant of a particular function $f:\mathcal{S}^2\to \mathcal{S}^2.$ Our function $f$ will be the composition of an isometry, a neural network, and a projection on $\mathcal{S}^2$. More precisely, our function $f$ is the composition of the following functions.

\begin{itemize}
\item An isometry $\zeta: \mathcal{S}^2\to X\subseteq \mathbb{R}^L$.

\item A neural network from $\mathbb{R}^L$ to $\mathbb{R}^3$, having $h\geq  1$ hidden layers with weights and biases $(W^1,b^1),\dots, (W^h,b^h)$ with ReLu activation function in between, and a linear layer $(W^{h+1},b^{h+1})$ at the end.

\item The last output of the neural network is projected to the unit sphere $\mathcal{S}^2.$

\end{itemize}

In order to get a Lipschitz constant for the projection on $\mathcal{S}^2$, we regularise the neural network in the way that its output does not contain a compact neighborhood of radius $\rho = 1$ of the origin.
We have the following proposition.

\begin{proposition}\label{Lipschitz_estimate}
Let $\vert\vert W^i\vert\vert$ $(i=1,\dots, h+1)$ be the operator norms of the linear layers of the neural network part of $f$, i.e.~ $\vert\vert W^i\vert\vert$ is the maximum eigenvalue of the square matrix $^tW^iW^i$. Assume that the range of the neural network part of $f$ does not contain a neighborhood of radius $\rho > 0$ of the origin. Then a Lipschitz constant of $f$ is given by
$$L_f = \frac{1}{\rho}\prod_{i=1}^{h+1}\vert\vert W^i\vert\vert.$$
\end{proposition}

\begin{proof}
See \cite{szegedy2013intriguing}
\end{proof}

All we need to do is compute the positive number $\rho$ specified in Proposition \ref{Lipschitz_estimate} i.e.~ we need to compute a positive lower bound for $\min_{x\in \mathcal{S}^2}\vert\vert E(x)\vert\vert$ where $E$ denotes the neural network part of $f.$ Note that there is no warranty that such a positive lower bound exists. However, with appropriate regularization, one can get a particular function $f$ such that the following computation works in practice.

\begin{proposition}
    Fix $n >0$. Then for all $x\in \mathcal{S}^2$, we have
\begin{equation}
 \min_{t\in \mathcal{V}_n} \{\vert\vert E(\zeta(\Delta^t(\mathbf{D}^0))\vert\vert\} \geq \min_{x\in \mathcal{S}^2} \vert\vert E(\zeta(x))\vert\vert  \geq \min_{t\in \mathcal{V}_n} \{\vert\vert E(\zeta(\Delta^t(\mathbf{D}^0))\vert\vert\} - L\; \frac{2\pi}{n},
\end{equation}
where $L$ is a Lipschitz constant of the neural network part $E.$

\end{proposition}
\begin{proof}
First, one can prove that $\max_{\tau\in \mathcal{F}} \{\mathrm{diam} (\Delta^{\tau}(\mathbb{D}^2))\}\leq \frac{2\pi}{n}.$
Secondly, the triangle inequality implies that for all $x\in \mathcal{S}^2$
\begin{equation}
    \begin{split}
        \vert\vert E(\zeta(x))\vert\vert &\geq \min_{t\in \mathcal{V}_n} \{\vert\vert E(\zeta(\Delta^t(\mathbf{D}^0))\vert\vert\} - L\; \max_{\tau\in \mathcal{F}} \{\mathrm{diam} (\Delta^{\tau}(\mathbb{D}^2))\}\\
        &\geq  \min_{t\in \mathcal{V}_n} \{\vert\vert E(\zeta(\Delta^t(\mathbf{D}^0))\vert\vert\} - L\; \frac{2\pi}{n},
    \end{split}
\end{equation}
and the result follows.
\end{proof}
\bibliographystyle{splncs04}
\bibliography{reference.bib}

\begin{thebibliography}{10}
\providecommand{\url}[1]{\texttt{#1}}
\providecommand{\urlprefix}{URL }
\providecommand{\doi}[1]{https://doi.org/#1}

\bibitem{alet2021noether}
Alet, F., Doblar, D., Zhou, A., Tenenbaum, J., Kawaguchi, K., Finn, C.: Noether
  networks: meta-learning useful conserved quantities. Advances in Neural
  Information Processing Systems  \textbf{34},  16384--16397 (2021)

\bibitem{bengio2013representation}
Bengio, Y., Courville, A., Vincent, P.: Representation learning: A review and
  new perspectives. IEEE transactions on pattern analysis and machine
  intelligence  \textbf{35}(8),  1798--1828 (2013)

\bibitem{blanco1997evaluation}
Blanco, M.A., Fl{\'o}rez, M., Bermejo, M.: Evaluation of the rotation matrices
  in the basis of real spherical harmonics. Journal of Molecular structure:
  THEOCHEM  \textbf{419}(1-3),  19--27 (1997)

\bibitem{brocker2013representations}
Br{\"o}cker, T., tom Dieck, T.: Representations of compact lie groups. Graduate
  Texts in Mathematics  (1985)

\bibitem{DBLP:journals/corr/BurdaGS15}
Burda, Y., Grosse, R.B., Salakhutdinov, R.: Importance weighted autoencoders.
  In: Bengio, Y., LeCun, Y. (eds.) 4th International Conference on Learning
  Representations, {ICLR} 2016, San Juan, Puerto Rico, May 2-4, 2016,
  Conference Track Proceedings (2016), \url{http://arxiv.org/abs/1509.00519}

\bibitem{burgess2018understanding}
Burgess, C.P., Higgins, I., Pal, A., Matthey, L., Watters, N., Desjardins, G.,
  Lerchner, A.: Understanding disentangling in beta-{VAE}. Learning
  Disentangled Representations: from Perception to Control Workshop, 2017
  (2017)

\bibitem{carbonneau2022measuring}
Carbonneau, M.A., Zaidi, J., Boilard, J., Gagnon, G.: Measuring
  disentanglement: A review of metrics. IEEE transactions on neural networks
  and learning systems  (2022)

\bibitem{caselles2019symmetry}
Caselles-Dupr{\'e}, H., Garcia~Ortiz, M., Filliat, D.: Symmetry-based
  disentangled representation learning requires interaction with environments.
  Advances in Neural Information Processing Systems  \textbf{32} (2019)

\bibitem{cha2023orthogonality}
Cha, J., Thiyagalingam, J.: Orthogonality-enforced latent space in
  autoencoders: an approach to learning disentangled representations. In:
  International Conference on Machine Learning. pp. 3913--3948. PMLR (2023)

\bibitem{chadebec2022data}
Chadebec, C., Thibeau-Sutre, E., Burgos, N., Allassonni{\`e}re, S.: Data
  augmentation in high dimensional low sample size setting using a
  geometry-based variational autoencoder. IEEE Transactions on Pattern Analysis
  and Machine Intelligence  \textbf{45}(3),  2879--2896 (2022)

\bibitem{davidson2018hyperspherical}
Davidson, T.R., Falorsi, L., De~Cao, N., Kipf, T., Tomczak, J.M.:
  Hyperspherical variational auto-encoders. 34th Conference on Uncertainty in
  Artificial Intelligence (UAI-18)  (2018)

\bibitem{ding2020guided}
Ding, Z., Xu, Y., Xu, W., Parmar, G., Yang, Y., Welling, M., Tu, Z.: Guided
  variational autoencoder for disentanglement learning. In: Proceedings of the
  IEEE/CVF conference on computer vision and pattern recognition. pp.
  7920--7929 (2020)

\bibitem{do2019theory}
Do, K., Tran, T.: Theory and evaluation metrics for learning disentangled
  representations. International Conference on Learning Representations, {ICLR}
  2020  (2020)

\bibitem{dubrovin2012modern}
Dubrovin, B.A., Fomenko, A.T., Novikov, S.P.: Modern geometry—methods and
  applications: Part II: The geometry and topology of manifolds, vol.~104.
  Springer Science \& Business Media (2012)

\bibitem{esmaeili2024topological}
Esmaeili, B., Walters, R., Zimmermann, H., van~de Meent, J.W.: Topological
  obstructions and how to avoid them. Advances in Neural Information Processing
  Systems  \textbf{36} (2024)

\bibitem{falorsi2018explorations}
Falorsi, L., De~Haan, P., Davidson, T.R., De~Cao, N., Weiler, M., Forr{\'e},
  P., Cohen, T.S.: Explorations in homeomorphic variational auto-encoding.
  ICML18 Workshop on Theoretical Foundations and Applications of Deep
  Generative Models  (2018)

\bibitem{de2018topological}
de~Haan, P., Falorsi, L.: Topological constraints on homeomorphic
  auto-encoding. NeurIPS 2018 workshop on Integration of Deep Learning Theories
   (2018)

\bibitem{harmonicsclaus}
Harmonics, S.: Claus mulier. Lecture Notes in Mathematics (LNM))  \textbf{17}
  (1966)

\bibitem{hatcher2005algebraic}
Hatcher, A.: Algebraic topology. Cambridge University Press (2005)

\bibitem{higgins2018towards}
Higgins, I., Amos, D., Pfau, D., Racaniere, S., Matthey, L., Rezende, D.,
  Lerchner, A.: Towards a definition of disentangled representations. arXiv
  preprint arXiv:1812.02230  (2018)

\bibitem{higgins2017beta}
Higgins, I., Matthey, L., Pal, A., Burgess, C.P., Glorot, X., Botvinick, M.M.,
  Mohamed, S., Lerchner, A.: beta-{VAE}: Learning basic visual concepts with a
  constrained variational framework. ICLR (Poster)  \textbf{3} (2017)

\bibitem{hoffman2016elbo}
Hoffman, M.D., Johnson, M.J.: Elbo surgery: yet another way to carve up the
  variational evidence lower bound. In: Workshop in Advances in Approximate
  Bayesian Inference, NIPS. vol.~1 (2016)

\bibitem{huh2024isometric}
Huh, I., Choe, J.M., KIM, Y., Kim, D., et~al.: Isometric quotient variational
  auto-encoders for structure-preserving representation learning. Advances in
  Neural Information Processing Systems  \textbf{36} (2024)

\bibitem{kaczynski2004computational}
Kaczynski, T., Mischaikow, K.M., Mrozek, M.: Computational homology, vol.~157.
  Springer (2004)

\bibitem{kim2018disentangling}
Kim, H., Mnih, A.: Disentangling by factorising. In: International conference
  on machine learning. pp. 2649--2658. PMLR (2018)

\bibitem{kingma2013auto}
Kingma, D.P., Welling, M.: Auto-encoding {V}ariational {B}ayes. In: Bengio, Y.,
  LeCun, Y. (eds.) 2nd International Conference on Learning Representations,
  {ICLR} 2014, Banff, AB, Canada, April 14-16, 2014, Conference Track
  Proceedings (2014), \url{http://arxiv.org/abs/1312.6114}

\bibitem{klushyn2019learning}
Klushyn, A., Chen, N., Kurle, R., Cseke, B., van~der Smagt, P.: Learning
  hierarchical priors in vaes. Advances in neural information processing
  systems  \textbf{32} (2019)

\bibitem{liu2020metrics}
Liu, X., Thermos, S., Valvano, G., Chartsias, A., O’Neil, A., Tsaftaris,
  S.A.: Metrics for exposing the biases of content-style disentanglement. arXiv
  preprint arXiv:2008.12378  \textbf{7} (2020)

\bibitem{locatello2019challenging}
Locatello, F., Bauer, S., Lucic, M., Raetsch, G., Gelly, S., Sch{\"o}lkopf, B.,
  Bachem, O.: Challenging common assumptions in the unsupervised learning of
  disentangled representations. In: international conference on machine
  learning. pp. 4114--4124. PMLR (2019)

\bibitem{menkovski2024small}
Menkovski, V., Portegies, J.W., Ravelonanosy, M.R.: Small time asymptotics of
  the entropy of the heat kernel on a riemannian manifold. Applied and
  Computational Harmonic Analysis p. 101642 (2024)

\bibitem{milnor1965topology}
Milnor, J.: Topology from the differentiable viewpoint, univ. Press of
  Virginia, Charlottesville  \textbf{1990} (1965)

\bibitem{ijcai2020p375}
Perez~Rey, L.A., Menkovski, V., Portegies, J.: Diffusion variational
  autoencoders. In: Bessiere, C. (ed.) Proceedings of the Twenty-Ninth
  International Joint Conference on Artificial Intelligence, {IJCAI-20}. pp.
  2704--2710. International Joint Conferences on Artificial Intelligence
  Organization (7 2020). \doi{10.24963/ijcai.2020/375},
  \url{https://doi.org/10.24963/ijcai.2020/375}, main track

\bibitem{rezende2014stochastic}
Rezende, D.J., Mohamed, S., Wierstra, D.: Stochastic backpropagation and
  approximate inference in deep generative models. In: International conference
  on machine learning. pp. 1278--1286. PMLR (2014)

\bibitem{sepliarskaia2019not}
Sepliarskaia, A., Kiseleva, J., de~Rijke, M.: How not to measure
  disentanglement. In: ICML Workshop on Theoretic Foundation, Criticism, and
  Application Trend of Explainable AI (July 2021)

\bibitem{sonderby2016ladder}
S{\o}nderby, C.K., Raiko, T., Maal{\o}e, L., S{\o}nderby, S.K., Winther, O.:
  Ladder variational autoencoders. Advances in neural information processing
  systems  \textbf{29} (2016)

\bibitem{suter2019robustly}
Suter, R., Miladinovic, D., Sch{\"o}lkopf, B., Bauer, S.: Robustly disentangled
  causal mechanisms: Validating deep representations for interventional
  robustness. In: International Conference on Machine Learning. pp. 6056--6065.
  PMLR (2019)

\bibitem{szegedy2013intriguing}
Szegedy, C., Zaremba, W., Sutskever, I., Bruna, J., Erhan, D., Goodfellow,
  I.J., Fergus, R.: Intriguing properties of neural networks. In: Bengio, Y.,
  LeCun, Y. (eds.) 2nd International Conference on Learning Representations,
  {ICLR} 2014, Banff, AB, Canada, April 14-16, 2014, Conference Track
  Proceedings (2014)

\bibitem{tomczak2018vae}
Tomczak, J., Welling, M.: Vae with a vampprior. In: International conference on
  artificial intelligence and statistics. pp. 1214--1223. PMLR (2018)

\bibitem{tonnaer2022quantifying}
Tonnaer, L., Rey, L.A.P., Menkovski, V., Holenderski, M., Portegies, J.:
  Quantifying and learning linear symmetry-based disentanglement. In:
  International Conference on Machine Learning. pp. 21584--21608. PMLR (2022)

\bibitem{van2019disentangled}
Van~Steenkiste, S., Locatello, F., Schmidhuber, J., Bachem, O.: Are
  disentangled representations helpful for abstract visual reasoning? Advances
  in neural information processing systems  \textbf{32} (2019)

\end{thebibliography}
\nocite{*}

\end{document}